\pdfoutput=1
\documentclass[reqno,letterpaper]{amsart}
\RequirePackage[OT1]{fontenc}
\usepackage{xspace}
\usepackage[usenames]{color}
\definecolor{SubtleColor}{rgb}{0,0,.50}
\usepackage{natbib}
\usepackage{hyperref}       
\hypersetup{colorlinks=true, allcolors=SubtleColor, pdfborder={0 0 0}}
\usepackage[capitalize,sort]{cleveref}
\crefname{appsec}{Appendix}{Appendices}
\usepackage{amsmath,amssymb,amsthm,bm,bbm,amsfonts,mathtools} 
\usepackage{url}            
\usepackage{graphicx}
\usepackage{booktabs}       
\usepackage[nice]{nicefrac}       
\usepackage{pifont}
\usepackage{microtype}      
\usepackage{algpseudocode, algorithm}
\usepackage{lipsum}
\usepackage{subfigure}
\usepackage{multirow}
\usepackage[toc, page]{appendix}

\newcommand{\data}{X} 
\newcommand{\param}{\bm{\theta}} 
\newcommand{\undens}{f} 
\newcommand{\D}{\mathcal{D}} 
\newcommand{\family}{\mathcal{H}} 
\newcommand{\post}{p_{\data}} 
\newcommand{\bth}{\bm{\theta}}

\DeclareMathOperator*{\KL}{\mathcal{D}_{\text{KL}}} 
\DeclareMathOperator*{\eKL}{\tilde{\mathcal{D}}_{\text{KL}}}
\DeclareMathOperator*{\E}{\mathbb{E}} 

\DeclareMathOperator*{\argmin}{arg\,min}
\DeclareMathOperator*{\argmax}{arg\,max}
\algdef{SE}[DOWHILE]{Do}{doWhile}{\algorithmicdo}[1]{\algorithmicwhile\ #1}
\newcommand{\cmark}{\ding{51}}%
\newcommand{\xmark}{\ding{55}}%

\newtheorem{theorem}{Theorem}

\begin{document}
\sloppy

\title{Boosting Variational Inference}

\author[F.~Guo]{Fangjian Guo}
\address{Massachusetts Institute of Technology}
\urladdr{http://richardkwo.net}
\email{guo@csail.mit.edu}

\author[X.~Wang]{Xiangyu Wang}
\address{Duke University}
\urladdr{http://www2.stat.duke.edu/~xw56/}
\email{xw56@duke.edu}

\author[K.~Fan]{Kai Fan}
\address{Duke University}
\urladdr{http://people.duke.edu/~kf96/}
\email{kai.fan@duke.edu}

\author[T.~Broderick]{Tamara Broderick}
\address{Massachusetts Institute of Technology}
\urladdr{http://www.tamarabroderick.com}
\email{tbroderick@csail.mit.edu}

\author[D.~Dunson]{David B.~Dunson}
\address{Duke University}
\urladdr{http://www2.stat.duke.edu/~dunson/}
\email{dunson@duke.edu}

\date{\today}
\begin{abstract}
Variational inference (VI) provides fast approximations of a Bayesian posterior in part because it formulates posterior approximation as an optimization problem: to find the closest distribution to the exact posterior over some family of distributions. For practical reasons, the family of distributions in VI is usually constrained so that it does not include the exact posterior, even as a limit point. Thus, no matter how long VI is run, the resulting approximation will not approach the exact posterior. We propose to instead consider a more flexible approximating family consisting of \emph{all possible finite mixtures} of a parametric base distribution (e.g., Gaussian). For efficient inference, we borrow ideas from gradient boosting to develop an algorithm we call \emph{boosting variational inference} (BVI). BVI iteratively improves the current approximation by mixing it with a new component from the base distribution family and thereby yields progressively more accurate posterior approximations as more computing time is spent. Unlike a number of common VI variants including mean-field VI, BVI is able to capture multimodality, general posterior covariance, and nonstandard posterior shapes.
\end{abstract}

\maketitle

\section{Introduction} \label{sec:intro}
Bayesian inference offers a flexible framework for learning with rich, hierarchical models of data and for coherently
quantifying uncertainty in unknown parameters through the posterior distribution.  However, for any moderately complex 
model, the posterior is intractable to calculate exactly and must be approximated.  
Variational inference (VI) has grown in popularity as a method for approximating the posterior since it is often fast even for large data sets.
VI is fast partly because it formulates posterior approximation as an optimization problem; the rough idea is to find the closest distribution to the exact posterior over some family of distributions.

Mean-field variational inference (MFVI) is a particularly widely-used variant of VI due in part to its simplicity. MFVI assumes that the approximating distribution factorizes across the parameters of the model, and this assumption leads to an efficient coordinate-ascent algorithm \citep{blei2016variational}. But this assumption also means that 
MFVI effectively cannot capture multimodality and that MFVI tends to underestimate the posterior covariance, sometimes drastically \citep{bishop2006pattern, wang2005inadequacy, turner2011two, rue2009approximate, mackay2003information}.
The linear response technique of \citet{giordano2015linear} and the full-rank approach within \citet{kucukelbir2015automatic} provide a correction to the covariance underestimation of MFVI in the unimodal case but 
do not address the multimodality issue. ``Black-box'' inference, as in \citet{ranganath2014black} and the mean-field approach within \citet{kucukelbir2015automatic},
focus on making the MFVI optimization problem easier for practitioners,
by avoiding tedious calculations, but they do not change the optimization objective of MFVI and therefore still
face the problems outlined here.

An alternative and more flexible class of approximating distributions for variational inference (VI) is the family of mixture models. 
Even if we consider only Gaussian component distributions, one can
find a mixture of Gaussians that is arbitrarily close to \emph{any} continuous probability density \citep{epanechnikov1969non, parzen1962estimation}.  \citet{bishop1998approximating, jaakkola1998improving} and \citet{gershman2012nonparametric}
have previously considered using approximating families with a fixed number of mixture components. Since the problem is non-convex, for practical inference, these authors
also employ a further approximation to the VI optimization objective, or impose constraints on the components (e.g., Gaussians with isotropic covariance and equal weights \citep{gershman2012nonparametric}). 
The resulting optimization algorithms
also suffer from some practical limitations; for instance, they are sensitive to initial values. 
For good performance, one may need rerun the algorithm for multiple different initializations and multiple choices for the number of components.

Finally, in all of the variants of VI discussed above, the family of approximating distributions does not, in general, include the exact posterior---even as a limit point. Thus, no matter how long VI is run, the resulting approximation will not approach the exact posterior.

We propose a new algorithm, \emph{boosting variational inference} (BVI), that addresses these concerns. Inspired by boosting, BVI starts with a single-component approximation and proceeds to add
a new mixture component at each step for a progressively more accurate posterior approximation. Thus, BVI allows us to trade off computing time for more statistical accuracy on the fly. BVI effectively considers a more flexible approximating family consisting of \emph{all possible finite mixtures} of a parametric base distribution (e.g., Gaussian). We will see that this flexibility allows us to capture multiple modes, general posterior covariance, and nonstandard posterior shapes. We highlight that a similar idea and treatment is developed in a concurrent and independent paper~\citet{miller2016variational}.

We start by defining our optimization problem in \Cref{sec:variational}. We review boosting and gradient boosting in \Cref{sec:boosting}
and develop the BVI algorithm in \Cref{sec:bvi}. We demonstrate the improved performance of our algorithm on a range of target posteriors, including those generated from both real and simulated data in \Cref{sec:experiments}.

\section{Variational inference and Gaussian mixtures} \label{sec:variational}
Suppose that we have observed data $\data$, and we identify a set of parameters $\param$ of interest: $\param \in \mathbb{R}^{d}$.
In the remainder, we do not make any assumptions on the form or space of the observed data $\data$.
A Bayesian model is specified with a 
prior $\pi(\param)$ and a likelihood $p(\data | \param)$. Bayes Theorem yields the posterior distribution,
\begin{equation}
p(\param | \data) = \frac{\pi(\param) p(\data | \param)}{p(\data)} \propto \pi(\param) p(\data | \param) =: \undens(\param),\label{eq:bayes}
\end{equation}
which may be considered the product of Bayesian inference and which expresses
the practitioner's knowledge of $\param$ after seeing the data $\data$.
While the unnormalized posterior $f$ in \cref{eq:bayes} is easy to evaluate, the normalizing constant $p(\data)$ usually involves an intractable integral. So an approximation is needed to use the normalized posterior distribution $\post(\param) := p(\param | \data)$
for any moderately complex model.
The posterior $\post(\param)$ may further be used to report a point estimate of $\param$
via a mean, to report an uncertainty estimate via a covariance, or to report some other posterior functional:
$\E_{p} g(\param) = \int \post(\param) g(\param) \mathrm{d} \param$ for some function $g(\cdot)$. For example, $g(\param) = \param$ yields the posterior mean.

One approach to approximating $\post(\param)$ is to find, roughly, the closest
distribution among some family of distributions $\family$, where the distributions in $\family$ are typically easier to work with;
e.g., the calculation of posterior functionals may be easier for distributions in $\family$.
More precisely, we choose
a discrepancy $\D$ between distribution $q(\param)$ and the exact posterior
$\post(\param)$. Then we approximate $\post(\param)$ with the distribution in $\family$ that
minimizes $\D$.
We assume $\D$ is non-negative in general and that $\D$ is zero if and only if $q = p_{\data}$.
In general, though, the optimum 
\begin{equation}
\D^\ast := \inf_{q \in \family} \D(q, \post)
\end{equation}
over some constrained family of distributions $\family$ may be strictly greater than zero.
Roughly, we expect a larger $\family$ to yield a lower $\D^\ast$ but at a higher computational cost.
When $\D(q, \post) = \KL(q || \post)$, the Kullback-Leibler (KL) divergence between $q$ and $\post$, this optimization problem is called
\emph{variational inference}~\citep{jordan1999introduction, wainwright2008graphical, blei2016variational}. 

A straightforward choice for $\family$ is some simple parametric family of known distributions, which we denote by 
\begin{equation}
\family_1 = \{h_{\bm{\phi}}(\param): \bm{\phi} \in \Phi\},
\end{equation}
where $\Phi$ is the space of feasible parameters. For example, we have $ \family_1 = \{ \mathcal{N}_{\bm{\mu}, \bm{\Sigma}}(\bth): \bm{\mu} \in \mathbb{R}^d, \bm{\Sigma} \in \mathbb{R}^{d \times d}, \bm{\Sigma} \text{ is positive semi-definite (PSD)} \} $ for a multivariate Gaussian approximation, and $\family_1 = \{\text{Gamma}(\theta_1; a, b) \cdot \text{Poisson}(\theta_2; \lambda): a, b, \lambda >0\}$ for a mean-field Poisson-Gamma approximation to $p_X(\theta_1, \theta_2)$.

A natural extension to $\family_1$ is the family of \emph{mixtures}. 
Let $\Delta_k$ denote the ($k$-1)-dimensional simplex: $\Delta_k = \{w \in \mathbb{R}^{k}: \sum_{j=1}^{k} w_j = 1, w_j \ge 0\ \text{ for all } j\}$.
Then, 
\begin{equation*}
\family_k = \{h: h(\param) = \sum_{j=1}^{k} w_j h_{\phi_j} (\param), \bm{w} \in \Delta_k, \bm{\phi} \in \Phi^k\}
\end{equation*}
is the set of all $k$-component mixtures over these base distributions, where $\bm{\phi} = (\bm{\phi}_1, \cdots, \bm{\phi}_k)$ is the collection of all component-wise parameters. 
$\family_{1}$, namely a single-component base distribution, is usually adopted for VI, including MFVI. While some authors \citep{bishop1998approximating, jaakkola1998improving, gershman2012nonparametric} have considered $\family_k$ for $k > 1$, these past works have focused on the case where the number of components $k$ is fixed and finite.

In this article, we consider an even larger family $\family_{\infty}$, which is the set of \emph{all finite mixtures}, namely
\begin{equation} \label{eq:family}
	\family_{\infty}
		= \bigcup_{k=1}^{\infty} \mathcal{H}_k.
\end{equation}
Note that $\family_{\infty}$ forms the \emph{convex hull} of the base distributions in $\family_1$~\citep{li1999mixture, rakhlin2006applications}. 
The family $\family_{\infty}$ is highly flexible; for example, if multivariate Gaussians are used for $\family_1$, one can find a member of $\family_{\infty}$ arbitrarily close to any continuous density~\citep{epanechnikov1969non}. 
\Cref{tab:family-comparison} further compares $\family_{\infty}$ with other common choices of family.
Our main contribution is to propose a novel and efficient algorithm for approximately solving the discrepancy minimization problem in $\family_{\infty}$.

\begin{table}[htbp]
\caption{Comparison of different families for VI.}
\label{tab:family-comparison}
\centering
\begin{tabular}{lccc}
\toprule
Family & Covariance & Multimodality & {Arbitrary} \\ 
  &   &   & {approximation} \\ \midrule
{Mean-field} & \xmark & \xmark & \xmark \\ 
{Full-rank $\mathcal{N}_{\bm{\mu}, \bm{\Sigma}}$} & \cmark & \xmark & \xmark \\ 
\textbf{$\mathcal{H}_k$} & \cmark & \cmark & \xmark \\ 
\textbf{$\mathcal{H}_{\infty}$} & \cmark & \cmark & \cmark \\ \bottomrule
\end{tabular}
\end{table}

\section{Boosting and Gradient Boosting} \label{sec:boosting}
For a general target distribution $\post(\param)$, we do not expect
to achieve zero, or infimum, discrepancy for any finite mixture. 
Rather, we expect to get progressively
closer to the infimum by increasing the mixture size. In the case
where the base distribution is Gaussian, we expect the discrepancy to approach
zero as we increase the mixture size.
More precisely, we consider a greedy, incremental procedure, as in \citet{zhang2003sequential},
to construct a sequence of finite mixtures $q_1, q_2, \cdots$ for each $q_t \in \mathcal{H}_t$. 
The quality of approximation can be measured with the excess discrepancy 
\begin{equation*}
\Delta \D(q_t) := \D(q_t, p_{\data}) - \mathcal{D}^{\ast} \geq 0,
\end{equation*}
and we want $\Delta \D(q_t) \rightarrow 0$ as $t \rightarrow \infty$.
Then, for any $\epsilon > 0$, we can always find a large enough $t$ such that $\Delta \D(q_t) \leq \epsilon$. 

In particular, we start with a single base distribution $q_1 = h_{\bm{\phi}_1}$ for some $\bm{\phi}_1 \in \bm{\Phi}$.
In practice, a crude initialization (e.g., $q_1 = \mathcal{N}(\bm{0}, c\bm{I})$) satisfying $\text{supp}(q_1) \subseteq \text{supp}(p_X)$ should suffice.
Iteratively, at each step $t = 2,3,\ldots$, let $q_{t-1}$ be the approximation from the previous step.
We form $q_{t}$ by mixing a new base distribution $h_t$ with weight $\alpha_t \in [0,1]$ together with $q_{t-1}$ with weight $(1-\alpha_t)$, i.e.,
\begin{equation}
q_{t} = (1 - \alpha_t) q_{t-1} + \alpha_t h_t.
\end{equation}
This approach is called \emph{greedy} \citep[Ch.~4]{rakhlin2006applications} if we choose (approximately) the optimal base distribution $h_t$ and weight $\alpha_t$ at each step $t$, so the resulting $q_t$ satisfies:
\begin{equation}
\mathcal{D}(q_t, p_{\data}) \leq \inf_{\substack{\bm{\phi} \in \bm{\Phi} \\0 \leq \alpha \leq 1}} \mathcal{D} ((1-\alpha) q_{t-1} + \alpha h_{\bm{\phi}}, p_\data) + \epsilon_t. \label{eq:greedy}
\end{equation}
That is, rather than choosing $\alpha_t, h_t$ to be exactly the optimal $\alpha, h_{\bm{\phi}}$ at any round, it is enough to choose $\alpha_t, h_t$ so that the discrepancy is within some non-negative sequence $\epsilon_t \searrow 0$ of optimality. 

At each step the approximation $q_t$ remains normalized by construction and takes the form of a mixture of base distributions.
The iterative updates are in the style of \emph{boosting} or \emph{greedy error minimization} \citep{freund1995desicion, freund1999short, friedman2000additive, li1999mixture}. 
Under convexity and strong smoothness conditions on $\D(\cdot, p_X)$, {\em Theorem II.1} of \citet{zhang2003sequential} guarantees that $\Delta \D(q_t)$ converges to zero at rate $O(1/t)$ if \cref{eq:greedy} is performed exactly.  
We verify that KL divergence satisfies these conditions, subject to regularity conditions on the $q_t$, in \cref{th:KL-consistency}.

\subsection{Gradient Boosting}
Let $\D(q) := \D(q,p_\data)$ as a shorthand notation. Rather than jointly optimizing $\D((1-\alpha_t)q_{t-1} + \alpha_t h_t)$
over $(\alpha_t, h_t)$, which is \emph{non-convex} and difficult in general, we consider nearly optimal choices in the spirit of \cref{eq:greedy}.
In particular, we choose $h_t$ first, in the style of gradient descent. Then we fix $h_t$ and optimize the corresponding weight $\alpha_t$; we show in \cref{sec:weight_boost} that choosing the optimal $\alpha_t$ is a convex problem. 

To find $h_t$, we use gradient boosting \citep{friedman2001greedy} and consider the {\em functional gradient} $\nabla \D(q)$ at the current solution $q=q_{t-1}$.
In what follows, we adopt the notation $\langle g, h \rangle = \int g(\bth) h(\bth) d\bth$ for functions $g$ and $h$. In its general form, gradient boosting considers minimizing a loss $\mathcal{L}(f)$ by perturbing around its current solution $f$, where $f(\cdot)$ is a function. Suppose we \emph{additively} perturb $f$ to $f + \epsilon \  h$ with a function $h(\cdot)$; then a Taylor expansion would yield
\begin{equation}
\mathcal{L}(f + \epsilon h) = \mathcal{L}(f) + \epsilon \ \langle h, g \rangle + o(\epsilon^2)
\end{equation}
as $\epsilon \searrow 0$.
The functional gradient can be defined from the resulting linear term: $\nabla \mathcal{L}(f) := g$.

In our case, we consider the perturbation from $q$ to $(1 - \epsilon) q + \epsilon \ h$ so that the solution remains normalized, where now $h$ is a distribution that describes the shape of the perturbation and $\epsilon \in [0,1]$ describes the size of the perturbation. Now, as $\epsilon \searrow 0$, a Taylor expansion yields
\begin{equation}
\begin{split}
\D((1 - \epsilon) q + \epsilon h) &= \D(q + \epsilon \ (h - q)) \\
& =\D(q) + \epsilon \ \langle h - q, g \rangle +  o(\epsilon^2), \label{eq:grad-boost-taylor}
\end{split}
\end{equation}
where $g$ is the functional gradient $\nabla \D(q) := g(\bth)$.
Roughly, for $q$ fixed in \cref{eq:grad-boost-taylor}, we will choose $h$ to minimize the inner product term $\langle h, g \rangle = \langle h, \nabla \D(q) \rangle $. I.e., as in gradient descent, we choose $h$ to ``match the direction'' of $ -\nabla \D(q)$, where $\D(q):=\D(q, p_X)$ is the discrepancy we want to minimize. We provide more detail in \cref{sec:laplace_boost}.

\section{Boosting Variational Inference} \label{sec:bvi}
Boosting variational inference (BVI) applies the framework of the previous section with Kullback-Leibler (KL) divergence as the discrepancy measure.  We first justify the choice of KL, and then present a two-stage multivariate Gaussian mixture
boosting algorithm (\cref{alg:laplacian}) to stochastically decrease the excess discrepancy. In each iteration, our algorithm first chooses $h_t$ via gradient boosting, and then solves for $\alpha_t$ given $h_t$. 

For $\undens$ as in \cref{eq:bayes}, the KL discrepancy is defined as 
\begin{align}
\D(q) = \KL(q \| p_\data) &=  \int q(\bth) \log \frac{q(\bth)}{p_\data(\bth)} \mathrm{d}\bth \label{eq:kl} \\
\nonumber
&= \log p(\data) + \int q(\bth) \log \frac{q(\bth)}{\undens(\bth)} \mathrm{d}\bth.
\end{align}
By dropping the term $\log p(\data)$, which is a constant in $q$, an \emph{effective discrepancy} can be defined as 
\begin{equation} \label{eq:effective_discrepancy}
\eKL(q) = \int q(\bth) \log \frac{q(\bth)}{\undens(\bth)}\mathrm{d}\bth,
\end{equation}
which is the negative value of the evidence lower bound, or ``ELBO''~\citep{blei2016variational}.

The following theorem shows that KL satisfies the greedy boosting conditions we discuss
after \cref{eq:greedy} \citep{zhang2003sequential, rakhlin2006applications}
under assumptions on $q$ that we examine after the proof.
These conditions then also hold for the effective discrepancy $\eKL(q)$. 

\begin{theorem}\label{th:KL-consistency}
Given densities $q_1, q_2$ and true density $p$, KL divergence is a convex functional, i.e., for any $\alpha \in [0, 1]$ satisfying
\begin{equation*}
\KL((1-\alpha) q_1 + \alpha q_2 \| p) \leq (1-\alpha) \KL(q_1 \| p) + \alpha \KL(q_2 \| p).
\end{equation*}
If we further assume that densities are bounded $q_1(\bth), q_2(\bth) \geq a > 0$, and denote the functional gradient of KL at density $q$ as $\nabla \KL(q) = \log q(\bth) - \log p(\bth)$, then the KL divergence is also strongly smooth, i.e., satisfying
\begin{align*}
\lefteqn{ \KL(q_2\|p) - \KL(q_1\|p) } \\
	& \leq \langle \nabla \KL(q_1\|p), q_2 - q_1 \rangle + \frac{1}{a}\|q_2 - q_1\|_2^2.
\end{align*}
\end{theorem}

\begin{proof}
For any densities $q_1, q_2$ and $\alpha \in [0,1]$, we have the convexity from
\begin{align*}
&\quad \KL((1-\alpha) q_1 + \alpha q_2 \| p) \\
& = (1 - \alpha) \int q_1 \log \frac{(1-\alpha) q_1 + \alpha q_2}{p} \mathrm{d} \bth + \alpha \int q_2 \log \frac{(1-\alpha) q_1 + \alpha q_2}{p} \mathrm{d} \bth \\
&= (1-\alpha) \int q_1 \log \frac{q_1 [1 + \alpha (q_2 / q_1 - 1) ]}{p} \mathrm{d} \bth + \alpha \int q_2 \log \frac{q_2 [1 + (1-\alpha) (q_1/q_2 - 1) ]}{p} \mathrm{d} \bth \\
&\leq (1-\alpha) \KL(q_1 \| p) + \alpha \KL(q_2 \| p),
\end{align*}
where we have used $\log (1 + x) \leq x$ and $\int (q_2 - q_1) \mathrm{d} \bth = 0$. 

Let $h = q_2 - q_1$; then we have $\int h(\bth) \mathrm{d}\bth = \int (q_2 - q_1)\mathrm{d}\bth = 0$. Again, using the inequality $\log (1 + x) \leq x$, we have the strong smoothness from 
\begin{align*}
&\quad { \KL(q_1 + h\|p) - \KL(q_1\|p) } \\
&= \langle q_1 + h, \log(q_1 + h) \rangle - \langle q_1, \log q_1 \rangle - \langle h, \log p \rangle\\
&\leq \langle q_1 + h, \frac{h}{q_1} + \log q_1 \rangle - \langle q_1, \log q_1 \rangle - \langle h, \log p \rangle\\
&= \langle h, \log q_1\rangle + \langle h, \frac{h}{q_1}\rangle - \langle h, \log p \rangle\\
&\leq \langle \log q_1 - \log p, h\rangle + \frac{1}{a}\|h\|_2^2,
\end{align*}
where in the last step we used that $q_1, q_2 \geq a > 0$. 
\end{proof}

Here we assume the densities $q_1, q_2$ are lower-bounded by $a > 0$. While this requirement is unrealistic for densities with unbounded support, it is not so onerous if the parameter space is bounded. Even when the parameter space is not intrinsically bounded, we are often most interested in a bounded subset of the parameter space. E.g., we might focus a compact set $\bm{\Theta} \subset \mathbb{R}^d$ such that $\int_{\bm{\Theta}} p_x(\bth) \mathrm{d} \bth = (1 - \epsilon)$ for a small $\epsilon > 0$. 

\subsection{Setting $\alpha_t$ with SGD} \label{sec:weight_boost}
First we show how to find $\alpha_t$ given a fixed $h_t$. Then in \cref{sec:laplace_boost} below, we will show how to find $h_t$. For now, we can take derivatives of effective discrepancy with respect to $\alpha_t$ to find
\begin{align}
\nonumber \frac{\partial \eKL(q_t)}{\partial \alpha_t} &= \int (h_t - q_{t-1}) \log \frac{(1-\alpha_t)q_{t-1} + \alpha_t h_t}{\undens} \mathrm{d}\bth \\
& = \E_{\bth \sim h_t} \gamma_{\alpha_t}(\bth) - \E_{\bth \sim q_{t-1}} \gamma_{\alpha_t}(\bth), \label{eqs:1st-order-deriv} \\
\frac{\partial^2 \eKL(q_t)}{\partial \alpha_t^2} &= \int \frac{(h_t - q_{t-1})^2}{(1-\alpha_t)q_{t-1} + \alpha_t h_t} \mathrm{d}\bth \geq 0, \label{eqs:2nd-order-deriv}
\end{align} 
where 
\begin{equation}
\gamma_{\alpha_t}(\bth) := \log \frac{(1-\alpha_t)q_{t-1}(\bth) + \alpha_t h_t(\bth)}{\undens(\bth)}. \label{eqs:gamma}
\end{equation}

By \cref{eqs:2nd-order-deriv}, $\eKL$ is a {\em convex} function of $\alpha_t$. 
Using \cref{eqs:1st-order-deriv}, we can estimate the gradient with Monte Carlo by drawing samples from $h_t$ and $q_{t-1}$. Then we use stochastic gradient descent (SGD) to solve for $\alpha_t$, as presented below in \Cref{alg:sgd}.

Since the gradient is stochastically estimated instead of exact, to ensure convergence we use a decaying sequence of step sizes $(b/k)$ in \cref{alg:sgd} that satisfy the Robbins-Monro conditions~\citep{robbins1951stochastic}. Note that we use only $1^{\text{st}}$-order methods here because unlike the gradient, estimating the $2^{\text{nd}}$ derivative in \cref{eqs:2nd-order-deriv} involves computation no longer in the logarithmic scale and is hence numerically unstable. 

\begin{algorithm}[H]
\caption{SGD for solving $\alpha_t$} \label{alg:sgd}
\begin{algorithmic}
\Require{current approximation $q_{t-1}(\bth)$, new component $h_t(\bth)$, product of prior and likelihood $\undens(\bth)$}, Monte Carlo sample size $n$, initial step size $b > 0$, and tolerance~$\epsilon > 0$
\State Initialize $ k \leftarrow 0$, $ \alpha_t^{(0)} \leftarrow 0 $
\Do
	\State Independently draw $ \{\bth^{(h)}_i\} \sim h_{t} $ and $ \{\bth^{(q)}_i\} \sim q_{t-1} $ for $ i=1,\cdots,n $
	\State Compute $\hat{\mathcal{D}}^{\prime}_{\text{KL}} = \frac{1}{n} \sum_i (\gamma_{\alpha_t}(\bth_i^{(h)}) - \gamma_{\alpha_t}(\bth_i^{(q)}) )$ by \cref{eqs:gamma}
	\State $ k \leftarrow k + 1$
	\State $\alpha_t^{(k)} \leftarrow \alpha_t^{(k-1)} - (b / k) \hat{\mathcal{D}}^{\prime}_{\text{KL}}$
	\State $\alpha_t^{(k)} \leftarrow \max(\min(\alpha_t^{(k)}, 1),0)$
\doWhile{$ | \alpha_t^{(k)} - \alpha_t^{(k-1)}| < \epsilon $}
\State \Return{$\alpha_t^{(k)}$}
\end{algorithmic}
\end{algorithm}

\subsection{Setting $h_t$ with Laplacian Gradient Boosting} \label{sec:laplace_boost}
It is difficult to solve for the best $h_{\bm{\phi}^{\ast}}$ in the optimization problem within \cref{eq:greedy}, so instead we use gradient boosting to identify a new component $h_t$ that will be good enough, in the sense of \cref{eq:greedy}.
As we suggest after \cref{eq:grad-boost-taylor}, we propose (roughly) to choose $h_t$ in the ``direction'' of the negative
functional gradient. We first take the Taylor expansion of the effective discrepancy $\eKL(q_t)$ around $\alpha_t \searrow 0$
\begin{equation}
\eKL(q_t) = \eKL(q_{t-1}) + \alpha_t \langle h_t, \log ({q_{t-1}} / {\undens}) \rangle  - \alpha_t \langle q_{t-1}, \log ({q_{t-1}} / {\undens}) \rangle + o(\alpha_t^2).
\end{equation}
When the new component contributes small weight $\alpha_t$, this result suggests that we might choose $h_t$ to minimize $\langle h_t, \log ({q_{t-1}} / {\undens}) \rangle$, where $\log ({q_{t-1}}/{\undens})$ is the \emph{functional gradient} $\nabla \eKL(q_{t-1})$.

However, a direct minimization of the inner product is ill-posed since $h_t$ will \emph{degenerate} to a point mass at the minimum of functional gradient. Instead we consider the regularized minimization
\begin{equation}
\hat{h}_t = \argmin_{h = h_{\phi}} \quad \langle h, - \log (f / q_{t-1}) \rangle + \frac{\lambda}{2} \log \| h \|_2^2
\end{equation}
with some $\lambda > 0$. The problem can be rewritten as
\begin{equation}
\hat{h}_t = \argmax_{h = h_{\phi}} \mathbb{E}_{\bth \sim h}  \log \frac{\undens(\bth)}{q_{t-1}(\bth)} - \lambda /2 \cdot \log \| h \|_2^2. \label{eqs:obj-grad-boost}
\end{equation}
For a Gaussian family with $h_{\phi}(\bth) = \mathcal{N}_{\bm{\mu}, \bm{\Sigma}}(\bth)$, we have
\begin{equation}
\hat{\bm{\mu}}_t, \hat{\bm{\Sigma}}_t = \argmax_{\bm{\mu}, \bm{\Sigma}} \mathbb{E}_{\bth \sim \mathcal{N}_{\bm{\mu},\bm{\Sigma}}} \log \frac{\undens(\bth)}{q_{t-1}(\bth)} + \lambda /4 \cdot\log |\bm{\Sigma}|, \label{eqs:obj-grad-boost-gaussian}
\end{equation}
where the log determinant term prevents degeneracy.

\cref{eqs:obj-grad-boost-gaussian}, though, is still a non-convex problem and requires approximation.
Although methods based on SGD and the reparameterization trick~\citep{kingma2013auto} can be applied, they are out of the scope of this article. 
Instead, we seek to develop a simple and efficient procedure based on off-the-shelf optimization routines. 
For this purpose, we interpret \cref{eqs:obj-grad-boost-gaussian} as follows. 

We say that $\log (\undens / q_{t-1})$ is the {\em residual log-density} corresponding to posterior approximation $q_{t-1}$. Ideally, we would have $q_{t-1} \propto \undens$, and the residual would be a flat plane. In general, though, the residual has peaks where posterior density is underestimated and basins where the posterior density is overestimated. Intuitively, the new component $h_t$ should be introduced to ``patch'' the region where density is currently underestimated. This idea is illustrated in \Cref{fig:schematic-algo-2}.

\begin{figure}[!htbp]
\begin{center}
\centerline{\includegraphics[width=0.90\columnwidth]{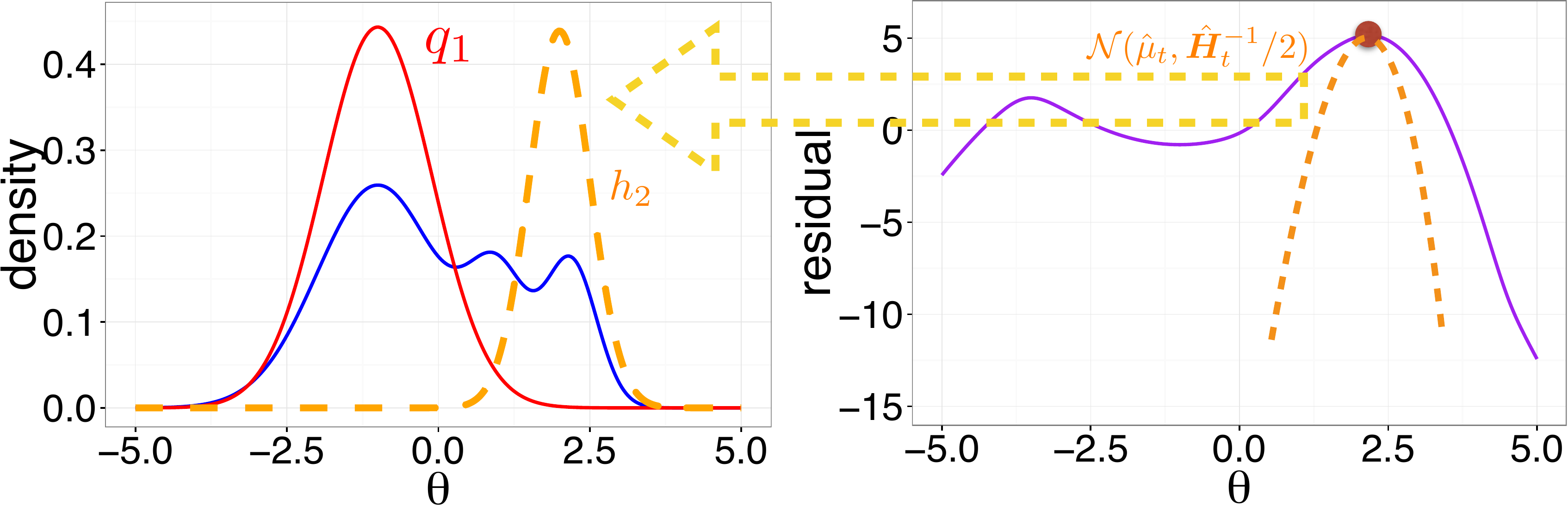}}
\caption{\Cref{alg:laplacian} identifies new component $h_2$ by finding a (local) peak of the log residual and its corresponding Hessian.}
\label{fig:schematic-algo-2}
\end{center}
\end{figure}

We therefore propose the following efficient heuristic (\cref{alg:laplacian}) to approximately optimize \cref{eqs:obj-grad-boost-gaussian}. We start by approximately decomposing the residual $\log (\undens / q_{t-1})$ into a constant plus a quadratic peak 
\begin{equation*}
\log (f(\bth) / q_{t-1}(\bth)) \approx - \frac{1}{2} (\bth - \bm\eta)^{T} \bm{H} (\bth - \bm{\eta}) + \text{const}.
\end{equation*}
For instance, we can choose $\bm\eta$ and $\bm{H}$ according to \emph{Laplace approximation} of $\undens / q_{t-1}$; that is, we match the location of a peak of $\log(f(\bth) / q_{t-1}(\bth))$ with $\eta$ and match the Hessian
at that location with $H$.
Then, \cref{eqs:obj-grad-boost-gaussian} reduces to
\begin{equation*}
\min_{\bm{\mu}, \bm{\Sigma}} -\frac{1}{2} \mathbb{E}_{\bth \sim \mathcal{N}(\bm{\mu}, \bm{\Sigma})} (\bth - \bm{\eta})^T \bm{H} (\bth - \bm{\eta}) + \frac{\lambda}{4} \log |\bm{\Sigma}|, 
\end{equation*}
which is identical to
\begin{equation}
\min_{\bm{\mu}, \bm{\Sigma}} -\frac{1}{2}\text{Tr}\left(\bm{H} (\bm{\mu}\bm{\mu}^T +\bm{\Sigma} - 2 \bm{\mu} \bm{\eta}^T) \right) + \frac{\lambda}{4} \log |\bm{\Sigma}|.
\end{equation}
This last equation describes a \emph{convex} problem, and we can write the solution in closed form:
\begin{equation}
\bm{\mu}^{\ast} = \bm{\eta}, \quad \bm{\Sigma}^{\ast} = \frac{\lambda}{2} \bm{H}^{-1}.
\end{equation}
For the Laplace approximation, the (local) optimum $\bm{\eta}$ and Hessian $\bm{H}$ can be calculated numerically.
This particular choice for $\bm{\eta}$ and $\bm{H}$, together with the resulting $h_t$, is described in
\Cref{alg:laplacian} and illustrated in \Cref{fig:schematic-algo-2}. 

\begin{algorithm}[H]
\caption{Laplacian Gradient Boosting for Gaussian Base Family $\mathcal{N}_{\bm{\mu}, \bm{\Sigma}}$} \label{alg:laplacian}
\begin{algorithmic}
\Require{evaluable product density of prior and likelihood $\undens(\bth)$ for $\bth \in \mathbb{R}^d$, number of iterations $T$}
\State Start with some initial approximation $ q_1= \mathcal{N}_{\bm{\mu}_1, \bm{\Sigma}_1}$ (e.g., $q_1 = \mathcal{N}_{\bm{0}, c\bm{I}}$)
\For{$t=2$ to $T$}
	\State Draw $ \bth_0 \sim q_{t-1}$
	\State $ \hat{\bm{\mu}}_t \leftarrow \argmin_{\bth} \log(q_{t-1}(\bth) / \undens(\bth)) $ with an optimization routine initialized at $ \bth_0$
	\State $ \bm{H}_t \leftarrow  \text{Hessian}_{\bth = \hat{\bm{\mu}}_t} \log(q_{t-1}(\bth) / \undens(\bth))$ with numerical approximation (e.g., finite difference)
	\State $ \hat{\bm{\Sigma}}_t \leftarrow \bm{H}_t^{-1} / 2$ 
	\State Let $ \hat{h}_t(\bth) = \mathcal{N}(\bth| \hat{\bm{\mu}}_t, \hat{\bm{\Sigma}}_t) $ be the new component
	\State $ \hat{\alpha}_t \leftarrow \argmin_{\alpha_t} \eKL((1-\alpha_t)q_{t-1} + \alpha_t  \hat{h}_t) $ with \cref{alg:sgd} 
	\State $ q_t \leftarrow (1-\hat{\alpha}_t) q_{t-1} + \hat{\alpha}_t \hat{h}_t $.
\EndFor
\State \Return $q_t$
\end{algorithmic}
\end{algorithm}

\subsubsection{Choosing $\lambda = 1$}
We now provide a justification for this choice based on a similar treatment of the KL divergence in the alternate direction, namely $\KL(p \| q)$, defined as 
\begin{align*}
{ \KL(p_{\data} \| q_t) } &= \int p_{\data}(\bth) \log \frac{p_{\data}(\bth)}{q_t(\bth)} \mathrm{d} \bth \\
&= \int p_{\data}(\bth) \log p_{\data}(\bth) \mathrm{d} \bth - c \int f(\bth) \log q_t(\bth) \mathrm{d} \bth,
\end{align*}
where $c = 1 / p(X)$ is the normalizing constant. 
Fixing $p_{\data}$, we use the shorthand $\D_{\textit{kl}}(q) := \KL (p_{\data} \| q)$, where
we employ the lower-case $\textit{kl}$ to distinguish this discrepancy from the KL divergence in the other direction (cf.\ \cref{eq:kl,eq:effective_discrepancy})
By dropping constants that do not involve $q_t$ as in \cref{eq:effective_discrepancy}, we can define an effective discrepancy, 
\begin{equation}
\tilde{\D}_{\textit{kl}}(q) = - \int f(\bth) \log q(\bth) \mathrm{d} \bth.
\end{equation}

We now consider a gradient boosting procedure for $\D_{\textit{kl}}$. 
Again, from a Taylor expansion around $\alpha_t \searrow 0$, we have
\begin{equation*}
\begin{split}
\tilde{\D}_{\textit{kl}}(q_t) &= \tilde{\D}_{\textit{kl}}(q_{t-1}) - \alpha_t \langle f(\bth) / q_{t-1}(\bth), h_t(\bth) \rangle \\
&\quad + \alpha_t \int f(\bth) \mathrm{d} \bth + o(\alpha_t^2). 
\end{split}
\end{equation*}
Here, we should match $h_t$ to the direction of $-\nabla  \tilde{D}_{\textit{kl}} = f / q_{t-1}$. To avoid degeneracy, we match in the $l_2$ norm
\begin{equation}
\min_{h = h_\phi, \lambda^{\prime} >0} \| \lambda^{\prime} h - f / q_{t-1} \|_2^2,
\end{equation}
as suggested by \citet{friedman2001greedy}. By plugging in the optimal value for $\lambda^{\prime}$, the objective is equivalent to 
\begin{equation}
\max_{h = h_\phi} \log \E_{\bth \sim h} \left (f(\bth) / q_{t-1}(\bth)\right) - 1/2 \cdot \log \|h\|_2^2,
\end{equation}
which, by Jensen's inequality, lower bounds the objective in \cref{eqs:obj-grad-boost} when $\lambda = 1$. Heuristically, we might expect that this bound \emph{iteratively tightens}; as $q_{t-1}(\bth) $ better approximates $p_{\data}(\bth)$, $f(\bth) / q_{t-1}(\bth) \propto p_{\data}(\bth) / q_{t-1}(\bth)$ approaches a constant function. Thus, referencing Jensen's inequality, we might expect $\E_{\bth \sim h} \log \left (f(\bth) / q_{t-1}(\bth)\right)$ to more closely approximate $\log \E_{\bth \sim h} \left (f(\bth) / q_{t-1}(\bth)\right)$.

\subsubsection{Stabilizing the Log-Residual} 
In \Cref{alg:laplacian}, we maximize the log-residual $ \log (f(\bth) / q_{t-1}(\bth))$ to find a local peak. If the exact posterior has a strictly heavier tail than any Gaussian random variable, we have that $\log ( f(\bth) / q_{t-1}(\bth) ) \rightarrow +\infty$ as $ \|\bth\| \rightarrow \infty$, and the argmax diverges. In order to ensure the Laplace approximation mean is finite, and therefore of practical use, in \cref{alg:laplacian}, we manually stabilize the tails of log-residual. In particular we add a small constant $a$ (e.g., $a = e^{-10}$) to the both $f$ and $q_{t-1}$, so that 
\begin{equation}
\log ((f(\bth) + a) / (q_{t-1}(\bth) + a)) \rightarrow 0 \quad \textrm{as } \|\bth\| \rightarrow \infty.
\end{equation}
The log-sum-exp trick is useful for implementing this modification.

\subsubsection{Scaling to High Dimensions}
When the dimension $d$ of the parameter space is high, the calculation of the covariance in \Cref{alg:laplacian} can be costly. In particular, the cost of estimating the Hessian with finite difference is $O(d^2)$, and another $O(d^3)$ for inversion or Cholesky decomposition. When $d$ is large, one option is to consider a diagonal approximation to the Hessian (and hence covariance). While this choice reduces the computation cost of a single iteration to $O(d)$, the convergence may be slower in terms of number of iterations.

\section{Experiments} \label{sec:experiments}
We refer to \cref{alg:laplacian} as \emph{boosting variational inference} (BVI) and evaluate its performance
on a variety of different exact posterior distributions---including distributions arising from both simulated and real data sets.
For comparison, we run
\emph{automatic differentiation variational inference} (ADVI) \citep{kucukelbir2015automatic},
which effectively performs mean-field variational inference (MFVI)
but does not require any conditional conjugacy properties in the generative model.
We use the ADVI implementation provided in \emph{Stan}~\citep{kucukelbir2015automatic, carpenter2016stan}.
In all experiments, we choose $n=100$ particles to evaluate stochastic gradients both in BVI and ADVI.
We choose a tolerance of $\epsilon = 10^{-4}$ in \Cref{alg:sgd}.
We use L-BFGS \citep{byrd1995limited} for optimization and the finite difference method when estimating the Hessian in \Cref{alg:laplacian}.

\subsection{Toy Examples}
\begin{figure}[!htbp]
\centering
\includegraphics[width=0.99\textwidth]{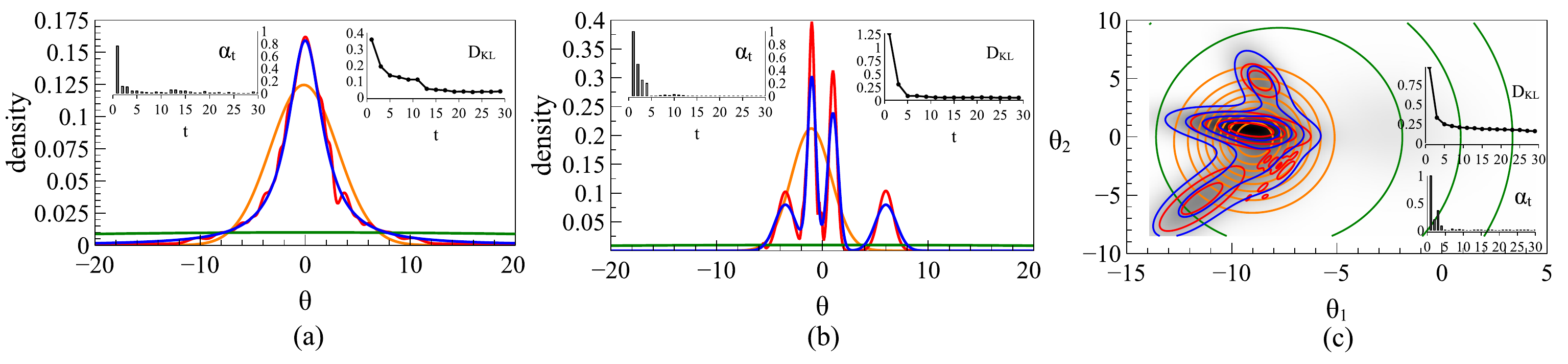}
\caption{Toy examples: (a) (heavy-tailed) Cauchy distribution (b) a mixture of four univariate Gaussians and (c) a mixture of five bivariate Gaussians with random means and covariances. Curves/contours are colored by blue (true), red (BVI), green (initial $q_1$ in BVI) and orange (ADVI). Sequence of $(\alpha_t)_{t}$ and Monte Carlo estimates of $(\KL(q_t))_{t}$ are plotted against iteration in subplots.}
\label{fig:toys}
\end{figure}

First, in \Cref{fig:toys}, we consider a number of simple choices for the exact target distribution $\post(\param)$ (plotted in blue) that illustrate the different behavior of ADVI (orange) and BVI (red).
In particular, the target distribution in (a) of \Cref{fig:toys} is a Cauchy density $\post(\param) \propto \frac{1}{1 + (\param/2)^2}$
with heavy tails. In \Cref{fig:toys}(b), we see that BVI is able to capture the multimodality of the target distribution, which is a mixture of univariate Gaussians with different locations and scales. 
And in \Cref{fig:toys}(c), the target distribution is a mixture of five two-dimensional Gaussians with different means and covariances.
In each case, we initialize BVI with a Gaussian with very large (co)variance (plotted in green), and then run BVI for 30 iterations. The sequences $(\alpha_t)_{t}$ and $(\KL(q_t \| p))_{t}$ are shown in subplots.

\begin{figure}[!htbp]
\centering
\includegraphics[width=0.7\textwidth]{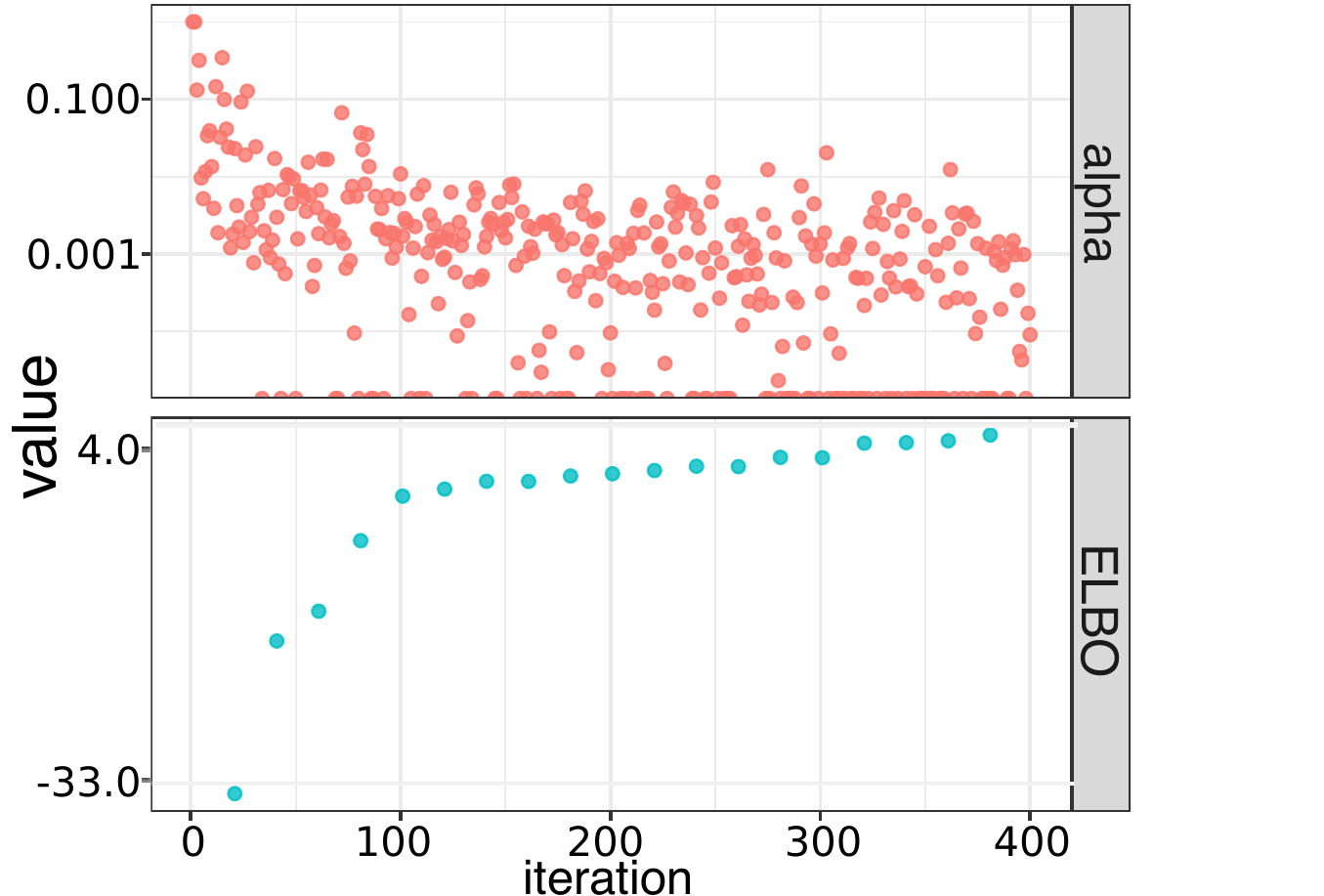}
\caption{Sequence of $(\alpha_t)_t$ and Monte Carlo ELBO estimates $(-\eKL(q_t))_t$ against the number of iterations, for estimating the ``banana''' distribution in \cref{sec:banana}. The $y$-axis is in logarithmic scale and the points of $\alpha_t$ at the bottom are values near zero.}
\label{fig:banana-trace}
\end{figure}

\begin{figure}[!htbp]
\centering
\includegraphics[width=0.7\textwidth]{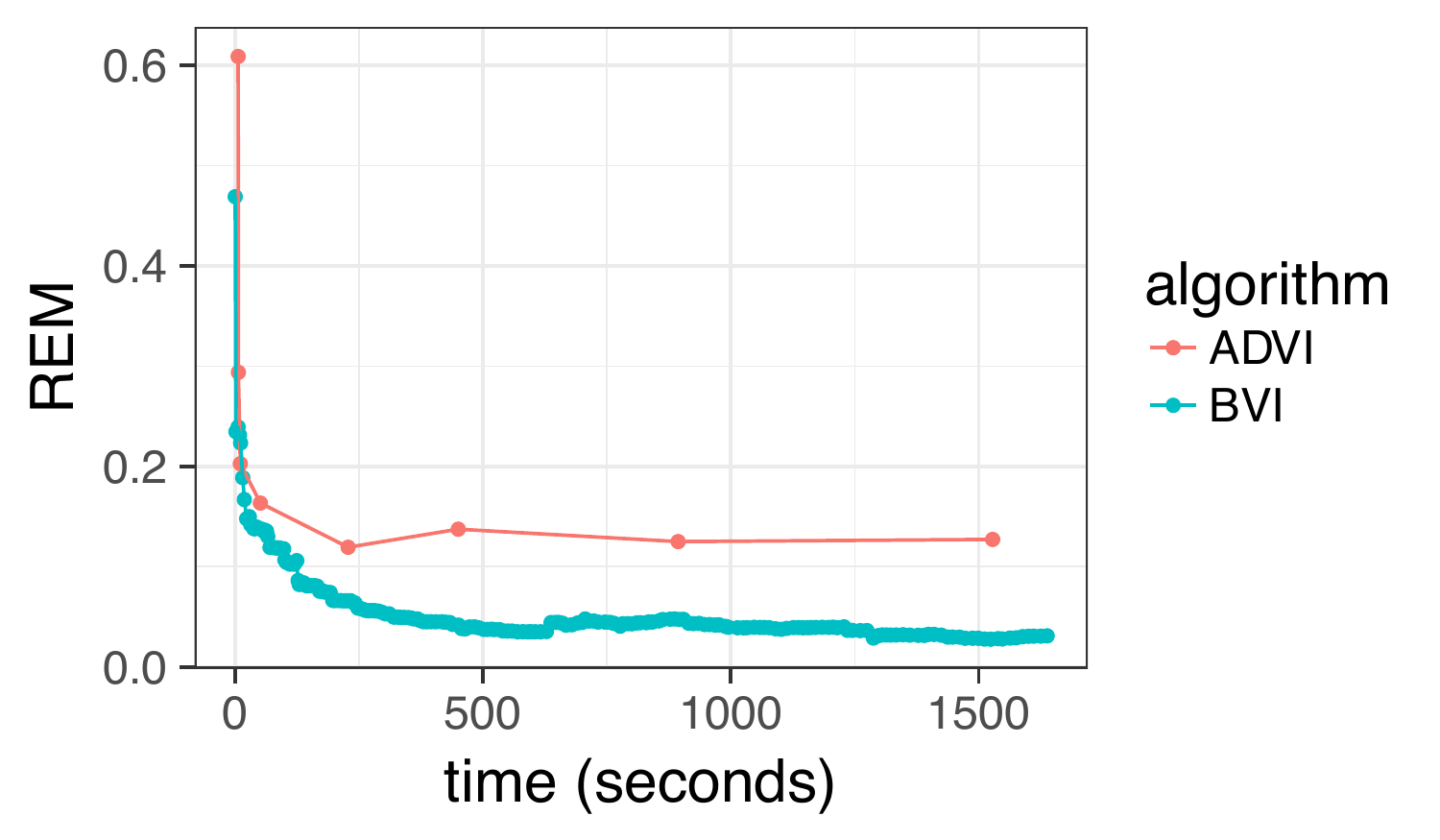}
\caption{Comparison of accuracy versus running time for different algorithms in sensor localization (\cref{sec:sensor}). Results are obtained by averaging 5 runs of each algorithm.}
\label{fig:sensor-timing}
\end{figure}

\subsection{Banana Distribution} \label{sec:banana}

\begin{figure}[!htbp]
\centering
\includegraphics[width=1.1\textwidth]{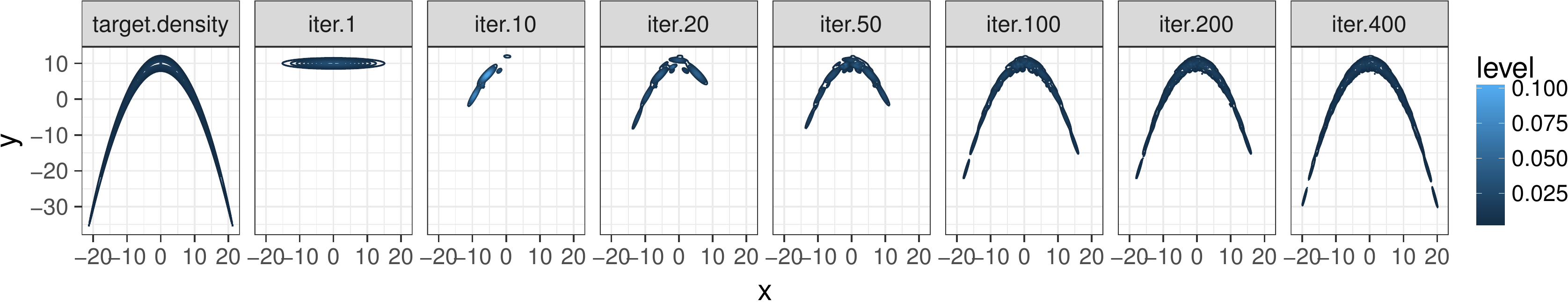}
\caption{Contour plots for the estimated posterior density of the ``banana'' distribution, as the number of iterations of BVI grows $(t=1, 10, 20, 50, 100, 200, 400)$. The true target density is shown in the leftmost panel.}
\label{fig:banana-density}
\end{figure}

We next highlight the ability of BVI to trade-off longer running times for more approximation accuracy on the fly as the number of iterations increase. We demonstrate this property on the ``banana'' distribution, which is widely-used for studying the performance of MCMC algorithms~\citep{haario1999adaptive, haario2001adaptive, roberts2009examples}. The distribution has the following density on $\mathbb{R}^2$:
\begin{equation*}
\post(\param_1, \param_2) \propto \exp(-\param_1^2/200 - (\param_2 + B \param_1^2 - 100 B)^2 / 2),
\end{equation*}
where the constant $B$ controls the curvature, and we choose $B = 0.1$. This target density is shown in the leftmost panel of \Cref{fig:banana-density}.

We initialize \Cref{alg:laplacian} with a standard Gaussian distribution and run for $T=400$ iterations. \Cref{fig:banana-density} shows the intermediate estimates of posterior density as the number of iterations grows. We can see how early, crude approximations are quickly made available, but then the approximation is refined over time. The sequence of weights $(\alpha_t)_t$ and Monte Carlo estimates of ELBO $(-\eKL(q_t))_t$ are shown in \Cref{fig:banana-trace}.

\subsection{Sensor Network Localization} \label{sec:sensor}
\begin{figure}[!htbp]
\includegraphics[width=0.9\textwidth]{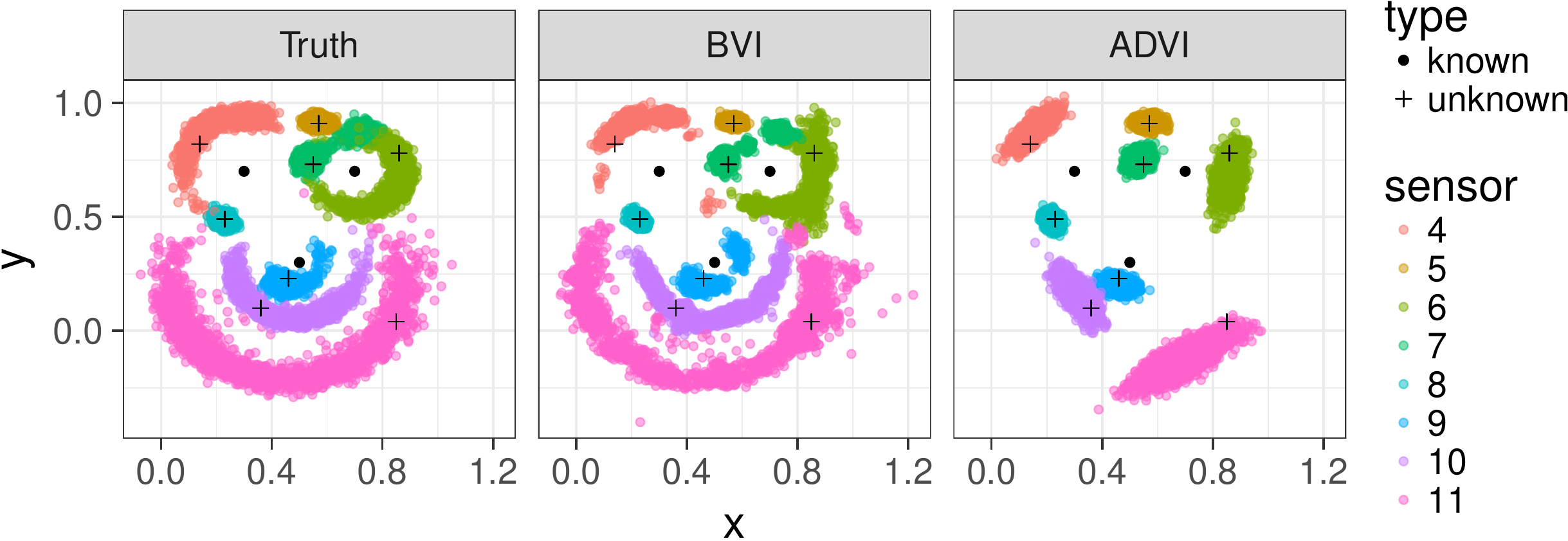}
\caption{Posterior samples of the sensor localization problem in \cref{sec:sensor} as inferred from NUTS (considered truth), BVI and ADVI. The circles (location known) and crosses (location unknown) mark the true coordinates of the sensors.}
\label{fig:sensor-posterior}
\end{figure}

In this experiment, we look at a sensor network localization problem that has been considered by \citet{ihler2005nonparametric, ahn2013distributed} and \citet{lan2014wormhole}. We have $N$ sensors located on a 2D plane, with coordinates $\bm{\theta}_i = (x_i, y_i)$ for $i=1,\cdots,N$. For each pair of sensors $(i, j)$, their distance is observed according to a Bernoulli distribution with the given probability:
\begin{equation}
Z_{i j} | \bm{\theta_i}, \bm{\theta_j} \sim \text{Ber}(\exp(-\| \bm{\theta}_i - \bm{\theta}_j \|_2^2 / (2 R^2))).
\end{equation}
Let $Y_{i j}$ denote the distance measured between sensors $i$ and $j$. When $Z_{i j} = 1$ (observed), the measured distance is the actual distance polluted by a Gaussian noise
\begin{equation}
Y_{i j} | Z_{i j} = 1, \bm{\theta}_i, \bm{\theta}_j \sim \mathcal{N}(\| \bm{\theta}_i - \bm{\theta}_j \|_2, \sigma^2).
\end{equation}
Otherwise, we assume $Y_{i j} = 0$ when $Z_{i j} = 0$ (unobserved).

Now, following the setup in \citet{ahn2013distributed} and \citet{lan2014wormhole}, we set $N=11$, $R = 0.3$ and $\sigma=0.02$. To avoid ambiguity of localization arising from translation and rotation, we assume $\bm{\theta}_i$ is known for $i = 1,2,3$. Hence, we are interested in  localizing the remaining 8 sensors. In particular, we want to infer the posterior distribution of $\{\bm{\theta}_i\}_{i=4}^{11}$ (located at crosses in \cref{fig:sensor-posterior}) given $\{Y_{i j}\}_{i \neq j}$ and $\{\bm{\theta}_i\}_{i=1}^3$ (located at circles in \cref{fig:sensor-posterior}). The resulting posterior distribution is a highly non-Gaussian and multimodal one on $\mathbb{R}^d$ with $d = 16$.

In \Cref{fig:sensor-posterior}, we compare the posterior samples as inferred by BVI and ADVI (full-rank). Both algorithms are given enough running time (200 iterations for BVI, 100,000 iterations for ADVI implemented in \emph{Stan}~\citep{kucukelbir2015automatic, carpenter2016stan}) to reach convergence. For reference, we treat the result from a converged NUTS sampler~\citep{hoffman2014no} as ground truth. In \Cref{fig:sensor-posterior}, we see that, unlike ADVI, our algorithm is able to reliably capture the complex shape in the posterior. 

We investigate the trade-off between time and accuracy for both BVI and ADVI. To measure the accuracy of posterior estimation, we use the relative error of the estimated mean (REM) as a summary index \citep{ahn2013distributed, lan2014wormhole}. REM measures the relative error of approximating the posterior mean in $l_1$ norm,
\begin{equation}
\text{REM}(q, p_X) = \frac{\| \E_{q} \bth - \E_{p_X} \bth\|_1}{\| \E_{p_X} \bth \|_1},
\end{equation}
and can be computed from sample averages. A comparison of REM versus running time for both BVI and ADVI is shown in \Cref{fig:sensor-timing}. Recall that we use $n=100$ particles for both BVI and ADVI to evaluate stochastic gradients. 

\subsection{Bayesian Logistic Regression}
\begin{figure}[!htbp]
\centering
\includegraphics[width=0.90\textwidth]{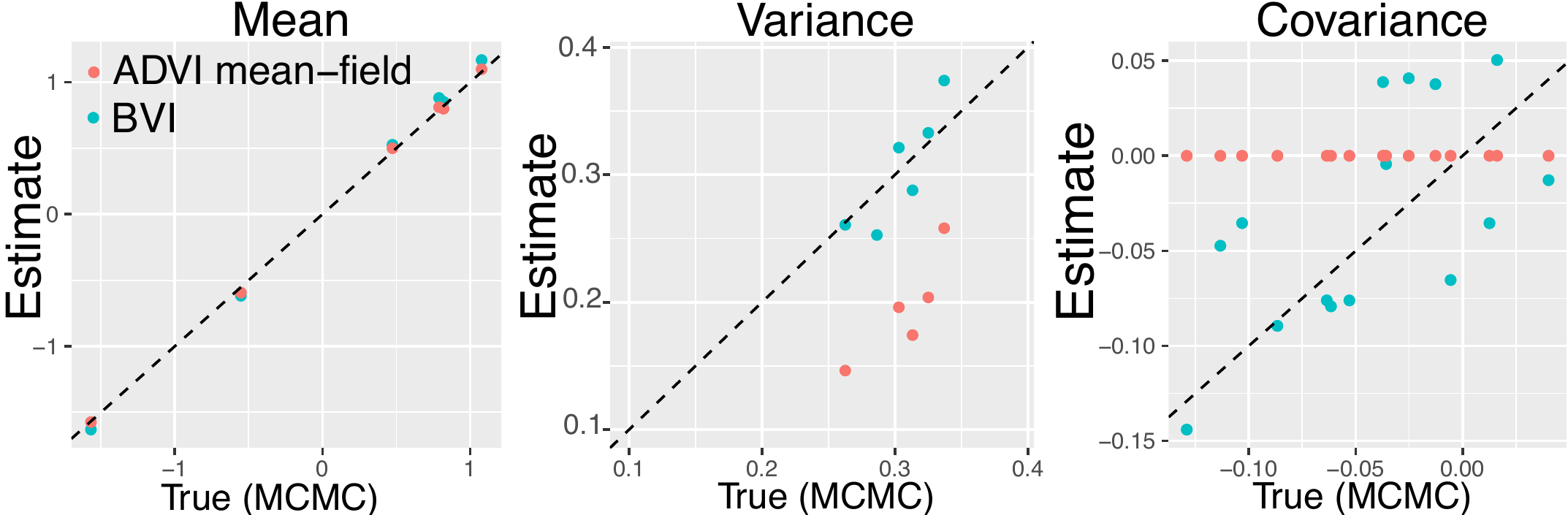}
\caption{Estimated posterior mean and covariance for logistic regression (dashed: estimate = true).}
\label{fig:logit-mean-cov}
\end{figure}

We apply our algorithm to Bayesian logistic regression for the \texttt{Nodal} dataset \citep{brown1980prediction}, consisting of $N=53$ observations of 6 predictors $\bm{x}_i$ (intercept included) and a binary response $y_i \in\{-1, +1\}$. The likelihood is $\prod_{i=1}^N g(y_i\mathbf{x}_i^\top\bm\beta)$, where $g(x)=(1+e^{-x})^{-1}$ and we use the prior $\bm{\beta} \sim \mathcal{N}(\bm{0}, \bm{I})$. We treat results from a Polya-Gamma sampler (a specialized MCMC algorithm) using R package \texttt{BayesLogit} \citep{polson2013bayesian} as the ground truth.
As shown in \Cref{fig:logit-mean-cov}, while both BVI and ADVI capture the correct mean, BVI provides better estimates of the variance and, unlike MFVI, does not set the covariances to zero. Here, though, we also see the limitations of BVI in the noisy covariance estimates. As expected, the case for BVI is more compelling in the previous examples, where MFVI yields biased estimates of the posterior means~\citep{turner2011two} or the posterior is multimodal.

\section{Conclusions} \label{sec:conclusions}
We have developed a new variational inference (VI) algorithm, \emph{boosting variational inference} (BVI). And we have demonstrated empirically that BVI can capture posterior multimodality and nonstandard shapes. There are notoriously few theoretical guarantees for quality of inference in VI. Nonetheless, the family considered here (\emph{all possible finite mixtures} of a parametric base distribution) may be more ripe for future theoretical analysis due to its desirable limiting properties, though such an analysis is outside the scope of the present paper.

\newpage
\bibliographystyle{abbrvnat}
\bibliography{bvi}

\end{document}